\documentclass{article}

\usepackage{times}
\usepackage{graphicx} 
\usepackage{subfigure} 

\usepackage{natbib}

\usepackage{algorithm}
\usepackage{algorithmic}

\usepackage{hyperref}

\usepackage{amsfonts,amsmath,amssymb,amscd,dsfont,mathrsfs,amsthm}

\usepackage{todonotes}

\usepackage[accepted]{icml2017}

\newtheorem{lemma}{Lemma}
\newtheorem{proposition}{Proposition}
\newtheorem{theorem}{Theorem}

\newcommand{\argmax}{\mathop{\rm argmax}}

\newcommand{\iprod}[2]{\langle #1, #2 \rangle}
\newcommand{\lmin}{\lambda_{\min}}
\newcommand{\norm}[1]{\left\|#1\right\|}
\newcommand{\mge}{\succeq}
\newcommand{\mi}{{-1}}
\newcommand{\R}{\mathbb{R}}\newcommand{\B}{\mathbb{B}}\newcommand{\Id}{\mathbf{I}}
\newcommand{\zerovec}{\mathbf{0}}
\newcommand{\onevec}{\mathbf{1}}
\newcommand{\defeq}{:=}
\newcommand{\secref}[1]{Section~\ref{#1}}
\newcommand{\lemref}[1]{Lemma~\ref{#1}}
\newcommand{\proref}[1]{Proposition~\ref{#1}}
\newcommand{\thmref}[1]{Theorem~\ref{#1}}
\newcommand{\assref}[1]{Assumption~\ref{#1}}
\newcommand{\eqnref}[1]{(\ref{#1})}
\newcommand{\algref}[1]{Algorithm~\ref{#1}}

\newcommand{\Prob}{\mathbb{P}}
\newcommand{\E}{\mathbb{E}}
\newcommand{\V}{\mathbb{V}}
\newcommand{\tr}[1]{\textrm{tr}\left(#1\right)}

\newcommand{\Evt}{\mathcal{E}}
\newcommand{\EvtG}{\Evt_G}
\newcommand{\EvtD}{\Evt_\Delta}
\newcommand{\EvtX}{\Evt_X}

\newcommand{\lihong}[1]{[[\textbf{LL:} #1]]}
\renewcommand{\lihong}[1]{}

\newcommand{\RN}[1]{%
  \textup{\uppercase\expandafter{\romannumeral#1}}%
}

\newtheorem{assumption}{Assumption}

\icmltitlerunning{Generalized Linear Contextual Bandits}

\begin{document} 

\twocolumn[
\icmltitle{Provably Optimal Algorithms for Generalized Linear Contextual Bandits}



\icmlsetsymbol{equal}{*}

\begin{icmlauthorlist}
\icmlauthor{Lihong Li}{msr}
\icmlauthor{Yu Lu}{yale}
\icmlauthor{Dengyong Zhou}{msr}
\end{icmlauthorlist}

\icmlaffiliation{msr}{Microsoft Research, Redmond, WA 98052}
\icmlaffiliation{yale}{Department of Statistics, Yale University, New Haven, CT, USA}

\icmlcorrespondingauthor{Lihong Li}{lihongli@microsoft.com}
\icmlcorrespondingauthor{Yu Lu}{yu.lu@yale.edu}
\icmlcorrespondingauthor{Dengyong Zhou}{denzho@microsoft.com}

\icmlkeywords{contextual bandit, exploration, generalized linear model}

\vskip 0.3in
]



\printAffiliationsAndNotice{}  

\begin{abstract}
Contextual bandits are widely used in Internet services from news recommendation to advertising, and to Web search. Generalized linear models (logistical regression in particular) have demonstrated stronger performance than linear models in many applications where rewards are binary. However, most theoretical analyses on contextual bandits so far are on linear bandits.  In this work, we propose an upper confidence bound based algorithm for generalized linear contextual bandits, which achieves an $\tilde{O}(\sqrt{dT})$ regret over $T$ rounds with $d$ dimensional feature vectors. This regret matches the minimax lower bound, up to logarithmic terms, and improves on the best previous result by a $\sqrt{d}$ factor, assuming the number of arms is fixed.
A key component in our analysis is to establish a new, sharp finite-sample confidence bound for maximum-likelihood estimates in generalized linear models, which may be of independent interest.  We also analyze a simpler upper confidence bound algorithm, which is useful in practice, and prove it to have optimal regret for certain
cases.
\end{abstract} 

\section{Introduction}
Contextual bandit problems are originally motivated by applications in clinical trials~\cite{woodroofe1979one}. When a standard treatment and a new treatment are available for a certain disease, the doctor needs to decide, in a sequetial manner, which of them to use based on the patient's profiles such as age, general physical status or medicine history. With the development of modern technologies, contextual bandit problems have more applications, especially in web-based recommendation, advertising and search~\cite{agarwal2009online, li2010contextual, liunbiased}. In the problem of personalized news recommendation, the website must recommend news articles that are most interesting to users that visit the website.  The problem is especially challenging for breaking news, as little data are available to make good prediction about user interest.
A trade-off naturally occurs in this kind of sequential decision making problems.  One needs to balance \emph{exploitation}---choosing actions that performed well in the past---and \emph{exploration}---choosing actions that may potentially give better outcomes.

In this paper, we study the following stochastic, $K$-armed contextual bandit problem. Suppose at each of the $T$ rounds, an agent is presented with a set of $K$ actions, each of which is associated with a context (a $d$-dimensional feature vector). By choosing an action based on the rewards obtained from previous rounds and on the contexts, the agent will receive a stochastic reward generated from some unknown distribution conditioned on the context and the chosen action. The goal of the agent is to maximize the expected cumulative rewards over $T$ rounds. The most studied model in contextual bandits literature is the linear model \cite{auer2003using,dani2008stochastic,rusmevichientong2010linearly,chu2011contextual,abbasi2011improved}, in which the expected rewards at each round is a linear combination of features in the context vector. The linear model is theoretically convenient to work with. However, in practice, we usually have binary rewards (click or not, treatment working or not).  Logistic regression model based algorithms have been shown to have substantial improvements over linear models~\cite{liunbiased}.  We therefore consider generalized linear models (GLM) in the contextual bandit setting, in which linear, logistic and probit regression serve as three important special cases. 

The celebrated work of \citet{lai1985asymptotically} first introduces the upper confidence bound (UCB) approach to efficient exploration. Later, the idea of confidence bound has been successfully applied to many stochastic bandits problems, from $K$-arm bandits problems \cite{auer2002finite, bubeck2012regret} to linear bandits~\cite{auer2003using,abbasi2011improved}. UCB-type algorithms are both efficient and provable optimal in $K$-arm bandits and $K$-armed linear bandits. However, most study are limited to the linear case. While some UCB-type algorithms using GLMs perform well empirically~\cite{liunbiased}, there is little theoretical study of them. A natural question arises: can we find an efficient algorithm to achieve the optimal convergence rate for generalized linear bandits?

\paragraph{Our Contributions} In this paper, we propose a GLM version of the UCB algorithm called SupCB-GLM that achieves a regret over $T$ rounds of order $\tilde{O}(\sqrt{dT})$. This rate improves the state-of-the-art results of \citet{filippi2010parametric} by a $\sqrt{d}$ factor, assuming the number of actions is fixed.  Moreover, it matches the GLM bandits problem's minimax lower bound indicated by the linear bandits problem and thus is optimal. SupCB-GLM  is inspired by the seminal work of \citet{auer2003using}, which introduced a technique to construct independence samples in linear contextual bandits.  A key observation in proving this result is that the $\ell_2$ confidence ball of the unknown parameter is insufficient to calculate a sharp upper confidence bound, yet what we need is the confidence interval in \emph{all} directions. Thus, we prove a finite sample normality type confidence bound for the maximum likelihood estimator of GLM. To the best of our knowledge, this is the first non-asymptotic normality type result for the GLM and might be of its own theoretical value. We also analyze a simple version of UCB algorithm called UCB-GLM that is widely used in practice. We prove it also achieves the optimal regret bound under a reasonable assumption. These results shed light on explaining the good empirical performance of GLM bandits in practice. 

\paragraph{Related Work} The study of GLM bandits problem goes back at least to \citet{sarkar1991one},
who considered \emph{discounted} regrets rather than cumulative regerts. They prove that a myopic rule without exploration is \emph{asymptotically} optimal. Recently,  \citet{filippi2010parametric} study the same stochastic GLM bandit problem considered here. They propose the GLM-UCB algorithm, similar to our \algref{alg: UCB}, which achieves a regret of ${\small \tilde{O}(d\sqrt{T})}$ after $T$ rounds. However, as we believe the optimal regret for stochastic GLM bandits should be the same as linear case when the number of actions is small, their rates misses a $\sqrt{d}$ term than the optimal rates.

Another line of research focuses on using EXP-type algorithms, which can be applied to almost any model classes~\cite{auer2002nonstochastic}. These algorithms, which choose actions using a carefully randomized policy, use importance sampling to reduce a bandit problem to its full-information analogue.
Later variants of the EXP4 algorithm~\cite{beygelzimer2010contextual,agarwal2014taming} give an $\tilde{O}(\sqrt{dKT})$ regret that is near-optimal with respect to $T$. However, these regret bounds have a $\sqrt{K}$ dependence. Moreover, these algorithms can be expensive to run: they either have a computational complexity exponential in $d$ for our GLM case, or need to make a large number of calls to a nontrivial optimization oracle.

\paragraph{Organization} 
\secref{sec:probsetting}  introduces the generalized linear bandit problem.  \secref{sec:glm} gives a brief review of the statistical properties of generalized linear model, and gives a sharp non-asymptotic normality-type result for GLM parameter estimation which can be of independent value.  With this tool, \secref{sec:alg} presents our algorithms and the main theoretical results.  \secref{sec:discussions} concludes the paper with further discussions, including several open problems.
All proofs are given in the supplementary materials.

\paragraph{Notations} 
For a vector $x \in \R^{d}$, we use $\norm{x}$ to denote its $\ell_2$- norm and $x'$ its transpose. 
$\B^d\defeq\{x\in\R^d ~:~ \norm{x}\le1\}$ is the unit ball centered at the origin.  The weighted $\ell_2$-norm associated with a positive-definite matrix $A$ is defined by $\norm{x}_A \defeq \sqrt{x' Ax}$. The minimum and maximum singular values of a matrix $A$ are written as $\lmin(A)$ and $\norm{A}$, respectively.  The trace of a matrix $A$ is $\tr{A}$.  For two symmetric matrices $A$ and $B$ of the same dimensions, $A \mge B$ means that $A-B$ is positive semi-definite.   For a real-valued function $f$, we use $\dot{f}$ and $\ddot{f}$ to denote its first and second derivatives.  Finally, $[n] \defeq \{1,2,\ldots,n\}$.

\section{Problem Setting} \label{sec:probsetting} 

We consider the stochastic $K$-armed contextual bandit problem.  Let $T$ be the number of total rounds. At round $t$, the agent observes a context consisting of a set of $K$ feature vectors, $\{x_{t,a} \mid a\in[K]\} \subset \R^d$, which is drawn IID from an unknown distribution $\nu$, with $\norm{x_{t,a}} \le 1$.  
Each feature vector $x_{t,a}$ is associated with an unknown stochastic reward  $y_{t,a} \in [0,1]$.  The agent selects one action, denoted $a_t$, and observes the corresponding reward $y_{t,a_t}$.  Finally, we make a regularity assumption about the distribution $\nu$: there exists a constant $\sigma_0 > 0$ such that $\lmin(\E[\frac{1}{K}\sum_{a\in[K]}x_{t,a}x_{t,a}']) \ge \sigma_0^2$ for all $t$.

In this paper, we are concerned with the generalized linear model, or GLM, in which there is an unknown $\theta^* \in \R^d$ and a fixed, strictly increasing \emph{link function} $\mu: \R \to \R$ such that $\E[Y \mid X] = \mu(X'\theta^*)$, where $X$ is the chosen action's feature and $Y$ the corresponding reward.
One can verify that linear and logistic models are special cases of GLM with $\mu(x)=x$ and $\mu(x)=1/(1+e^{-x})$, respectively. 

The agent's goal is to maximize the cumulative expected rewards over $T$ rounds. Suppose the agent takes action $a_t$ at round $t$. Then the agent's strategy can be evaluated by comparing its expected reward to the best expected reward. To do so, define the optimal action at round $t$ by $a_t^*=\argmax_{a \in [K]} \mu(x_{t,a}'\theta^*)$. Then, the agent's total regret of following strategy $\pi$ can be expressed as follows
\vspace{-1mm}
\begin{equation*}
R_T(\pi) \defeq \sum_{t=1}^{T} \left( \mu(x_{t,a_t^*}'\theta^*) - \mu(x_{t,a_t}'\theta^*) \right).
\end{equation*}
Note that $R_T(\pi)$ is in general a random variable due to the possible randomness in $\pi$. Denote by $X_{t}=x_{t,a_t}$, $Y_{t}=y_{t,a_t}$, and our model can be written as 
\begin{equation} \label{eq:model1}
Y_{t} = \mu(X_t'\theta^*) + \epsilon_t\,,
\end{equation}
where $\{\epsilon_t, t \in [T]\}$ are independent zero-mean noise. Here, $X_t$ is a random variable because the agent chooses current action based on previous rewards. Formally, we assume there is an increasing sequence of sigma fields $\{\mathcal{F}_{n}\}$ such that 
$\epsilon_t$ is $\mathcal{F}_{t}$-measurable with $\E \left[\,\epsilon_t \mid \mathcal{F}_{t-1}\,\right]=0$.  An example of $\mathcal{F}_{n}$ will be the sigma-field generated by $\{X_1,Y_1,\ldots,X_n, Y_n\}$. Also, we assume the noise $\epsilon_t$ is sub-Gaussian with parameter $\sigma$, where $\sigma$ is some positive, universal constant; that is, for all $t$,
\begin{equation} \label{eq:glm2}
\E \left [\, e^{\lambda \epsilon_t} \mid \mathcal{F}_{t-1}\,\right] \le e^{\lambda^2 \sigma^2/2}.
\end{equation}
In practice, when we have bounded reward $Y_t \in [0,1]$,
the noise $\epsilon_t$ is also bounded and hence satisfies \eqnref{eq:glm2} with some appropriate $\sigma$ value. In addition to the boundedness assumption on the rewards and feature vectors, we also need the following assumption on the link function $\mu$. 

\begin{assumption} \label{ass:kappa}
$\kappa := \inf_{\{\norm{x}\le 1,\,\,\norm{\theta-\theta^*} \le 1\}} \dot{\mu}(x'\theta)  > 0 $.
\end{assumption}

As we shall see in \secref{sec:glm}, the asymptotic normality of maximum-likelihood estimates implies the necessity of this assumption. Note that this assumption is weaker than Assumption 1 in \citet{filippi2010parametric}, as it only requires to control the \emph{local} behavior of $\dot{\mu}(x'\theta)$ near $\theta^*$. 

\begin{assumption} \label{ass:smooth}
$\mu$ is twice differentiable. Its first and second order derivatives are upper-bounded by $L_{\mu}$ and $M_\mu$, respectively.
\end{assumption}

It can be verified that \assref{ass:smooth} holds for the logistic link function, where we may choose $L_{\mu}=M_\mu=1/4$.

\section{Generalized Linear Models} \label{sec:glm}

To motivate the algorithms proposed in this paper, we first briefly review the classical likelihood theory of generalized linear models. In the canonical generalized linear model \cite{mccullagh1989generalized}, the conditional distribution of $Y$ given $X$ is from the exponential family, and its density, parameterized by $\theta\in\Theta$, can be written as
\begin{equation}
\Prob(Y \mid X) = \exp\left \{\frac{Y X'\theta^* - m(X'\theta^*)}{g(\eta)} + h(Y,\eta) \right\}. \label{eq:glm1}
\end{equation}
Here, $\eta \in \R^+$ is a known scale parameter; $m$, $g$ and $h$ are three normalization functions mapping from $\R$ to $\R$. The exponential family \eqnref{eq:glm1} is a very broad family of distributions including the Gaussian, binomial, Poisson, gamma and inverse-Gaussian distributions. It follows from standard properties of exponential families \cite{brown1986fundamentals} that $m$ is infinitely differentiable satisfying $\dot{m}(X'\theta^*)=\E [\,Y \mid X\,] = \mu(X'\theta^*)$ and $\ddot{m}(X'\theta^*)=\V(Y \mid X)$. It can be checked that the data generated from \eqnref{eq:glm1} automatically satisfies the sub-Gaussian condition \eqnref{eq:glm2}.

Suppose we have independent samples of $Y_1, Y_2, \ldots, Y_n$ condition on $X_1,X_2,\ldots,X_n$. The log-likelihood function of $\theta$ under model \eqnref{eq:glm1} is
\begin{eqnarray*}
\log \ell(\theta) &=& \sum_{t=1}^n \left[\frac{Y_t X_t'\theta - m(X_t'\theta)}{v(\eta)} + c(Y_t,\eta)\right] \\
&=&
\frac{1}{v(\eta)} \sum_{t=1}^{n} \left[ Y_t X_t'\theta - m(X_t'\theta) \right] + \mathrm{constant}\,.
\end{eqnarray*}
Consequently, the maximum likelihood estimate (MLE) may be defined by
\[
\hat{\theta}_n \in \argmax_{\theta\in\Theta} \sum_{t=1}^{n} \left[ Y_t X_t'\theta - m(X_t'\theta) \right]\,.
\]
From classical likelihood theory~\cite{lehmann1998theory}, we know that when the sample size $n$ goes to infinity, the MLE $\hat{\theta}_n$ is asymptotically normal, that is,
$\hat{\theta}_n - \theta^* \to \mathcal{N}(0, \mathcal{I}^{-1}_{\theta^*})$, where $\mathcal{I}_\theta= \sum_{t=1}^{n} \dot{\mu}(X_t'\theta)X_tX_t'$ is the Fisher Information Matrix. Note that if $\dot{\mu}(X_t'\theta^*) \to 0$, the asymptotic variance of $x'\hat{\theta}$ can go to infinity for some $x$. This suggests the necessity of \assref{ass:kappa}.

As we will see later, the normality result is crucial in our regret analysis of GLM bandits. However, to the best of our knowledge, there is no non-asymptotic normality results of the MLE for GLM. In the following, we present a finite-sample version of the classical asymptotic normality results, which can be of independent interest.
 
\begin{theorem} \label{thm:glm_ucb}
Define $V_n = \sum_{t=1}^{n} X_t X_t'$, and let $\delta>0$ be given.   Furthermore, assume
that
\begin{equation}
\lmin(V_n) \ge  \frac{512 M_\mu^2 \sigma^2}{\kappa^4} \left(d^2 + \log\frac{1}{\delta}\right)\,.
\label{eq:lmin1}
\end{equation} 
Then, with probability at least $1-3\delta$,  the maximum-likelihood estimator satisfies, for any $x \in \R^d$, that
\begin{equation} \label{eq:normality}
|x'(\hat{\theta}_n-\theta^*)| \le \frac{3\sigma}{\kappa} \sqrt{\log(1/\delta)} \norm{x}_{V_n^{-1}}\,.
\end{equation}
\end{theorem}
This theorem characterizes the behavior of MLE on \emph{every} direction. It implies that $x'(\hat{\theta}_n-\theta^*)$ has a sub-Gaussian tail bound for any $x \in \R^d$. It also provides a rigorous justification of the asymptotic upper confidence bound derived heuristically by \citet[Section~4.2]{filippi2010parametric}.

The proof of the theorem is given in the appendix.  It consists of two main steps, as is typical for proving normality-type results of MLEs~\citep{van2000asymptotic}.  We first show the $n^{-1/2}$-consistency of $\hat{\theta}$ to $\theta^*$. Then, by using a second-order Taylor expansion or Newton-step, we can prove the desired normality of $\hat{\theta}$.

The condition \eqnref{eq:lmin1} on $\lmin(V_n)$ is necessary for the consistency of estimating linear models~\cite{lai1982least,bickel2009simultaneous} and generalized linear models~\cite{fahrmeir1985consistency,chen1999strong}. It can be satisfied under mild conditions
such as the proposition below, which will be useful for our analysis.

\begin{proposition}
\label{pro:spectral}
Define $V_n=\sum_{t=1}^n X_t X_t'$, where $X_t$ is drawn iid from some distribution $\nu$ with support in the unit ball, $\B^d$.  Furthermore, let $\Sigma\defeq\E[X_tX_t']$ be the second moment matrix, and $B$ and $\delta>0$ be two positive constants.  Then, there exist positive, universal constants $C_1$ and $C_2$ such that $\lmin(V_n) \ge B$ with probability at least $1-\delta$, as long as
\begin{eqnarray*}
n &\ge& \left(\frac{C_1 \sqrt{d} + C_2 \sqrt{\log(1/\delta)}}{\lmin(\Sigma)}\right)^2 + \frac{2B}{\lmin(\Sigma)}\,.
\end{eqnarray*}
\end{proposition}

\begin{proof}[Proof Sketch]
We give a proof sketch here, and the full proof is found in the appendix.
In the following, for simplicity, we will drop the subscript $n$ when there is no ambiguity.  Therefore, $V_n$ is denoted $V$ and so on.
We will need a technical lemma, which is an existing result in random matrix theory. The version we presented here is adapted from Equation (5.23) of Theorem 5.39 from \citet{vershynin2010introduction}.
\begin{lemma}
\label{lem:rmsv}
Let $A\in\R^{n\times d}$ be a matrix whose rows $A_i$ are independent sub-Gaussian isotropic random vectors in $\R^d$ with parameter $\sigma$, namely, $\E \exp(x' (A_i - \E A_i)) \le \exp(\sigma^2 \norm{x}^2/2)$ for any $x \in \R^d$.  Then, there exist positive, universal constants $C_1$ and $C_2$ such that, for every $t \ge 0$, the following holds with
probability at least $1-2\exp(-C_2t^2)$, where $\varepsilon = \sigma^2 (C_1 \sqrt{d/n} + t/\sqrt{n})$: $\left\|\frac{1}{n} A'A - \mathbf{I}_d \right\| \le \max\{\varepsilon, \varepsilon^2\}\,.$
\end{lemma}

Let $X$ be a random vector drawn from the distribution $\nu$.  Define $Z\defeq \Sigma^{-1/2}X$.  Then $Z$ is isotropic, namely, $\E[ZZ'] = \Id_d$.  Define $U=\sum_{t=1}^n Z_tZ_t' = \Sigma^{-1/2}V\Sigma^{-1/2}$.
From \lemref{lem:rmsv}, we have that, for any $t$, with probability at least $1-2\exp(-C_2t^2)$, $\lmin(U)\ge n - C_1 \sigma^2 \sqrt{nd} - \sigma^2 t \sqrt{n}$,
where $\sigma$ is the sub-Gaussian parameter of $Z$, which is upper-bounded by $\norm{\Sigma^{-1/2}} = \lmin^{-1/2}(\Sigma)$ (see, e.g., \citet{vershynin2010introduction}).  
We thus can rewrite the above inequality (which holds with probability $1-\delta$ as
\[
\lmin(U) \ge n - \lmin^{-1}(\Sigma) \left(C_1 \sigma^2 \sqrt{nd} + t \sqrt{n} \right) \,.
\]
This implies the following lower bound:
\[
\lmin(V) \ge \lmin(\Sigma) n - C_1 \sqrt{nd} - C_2 \sqrt{n \log(1/\delta)}\,.
\]
Finally, simple calculations show that the last expression is no less than $B$ as long as $n$ is no smaller than the expression stated in the proposition, finishing the proof.
\end{proof}

\section{Algorithms and Main Results} \label{sec:alg}

In this section, we are going to present two algorithms. While the first algorithm is computationally more efficient, the second algorithm has a provable optimal regret bound.

\subsection{Algorithm UCB-GLM} \label{sec:UCB-GLM}

The idea of upper confidence bounds (UCB) is highly effective in dealing with the exploration and exploitation trade-off in many parametric bandit problems, including $K$-arm bandits~\cite{auer2002finite} and linear bandits~\cite{abbasi2011improved,auer2003using,chu2011contextual,dani2008stochastic}. 
For the generalized linear model considered here, since $\mu$ is a strictly increasing function, our goal is equivalent to choosing $a \in [K]$ to maximize $x_{t,a}'\theta^*$ at round $t$. Suppose $\hat{\theta}_t$ is our current estimator of $\theta^*$ after round $t$. An exploitation action is to take the action that maximizes the estimated mean value, while an exploration action is to choose the one that has the largest variance. Thus, to balance exploitation and exploration, we can simply choose the action that maximizes the sum of estimated mean and variance, which can be interpreted as an upper confidence bound of $x_{t,a}'\hat{\theta}_t$.  This leads to the algorithm UCB-GLM (\algref{alg: UCB}).

\begin{algorithm}
\caption{UCB-GLM}\label{alg: UCB}
\vspace{0.05in}
\textbf{Input}: the total rounds $T$, tuning parameter $\tau$ and $\alpha$. \vspace{0.00in}

\textbf{Initialization}: randomly choose $a_t \in [K]$ for $t \in [\tau]$, set $V_{\tau+1}=\sum_{i=1}^{\tau} X_{t}X_{t}'$ \\ \vspace{0.05in}
\textbf{For} {$t=\tau+1,\tau+2,\ldots,T$} \textbf{do} \vspace{-0.15in}
 \begin{itemize}
 \item[1.] Calculate the maximum-likelihood estimator $ \hat{\theta}_t$ by solving the equation \vspace{-3mm}
 \begin{equation} \label{eq:algquasi}
 \sum_{i=1}^{t-1} (Y_i - \mu(X_i'\theta)) X_i =0
 \end{equation}\vspace{-5mm}
 \item[2.] Choose $a_t = \argmax_{a \in [K]} \left ( X_{t,a}'\hat{\theta}_t + \alpha \norm{X_{t,a}}_{V_{t}^{-1}} \right)$ \vspace{-5mm}
 \item[3.] Observe $Y_t$, let $X_t \leftarrow X_{t,a_t}$, $V_{t+1} \leftarrow V_t+X_{t}X_{t}'$
 \end{itemize} \vspace{-2mm}
\textbf{End For} \vspace{0.05in}
\end{algorithm}

UCB-GLM take two parameters. At the initialization stage, we randomly choose actions to ensure a unique solution of \eqnref{eq:algquasi}.  The choice of $\tau$ in the theorem statement follows from \proref{pro:spectral} with $B=1$.  It should be noted that the IID assumption about contextual (i.e., the distribution $\nu$) is only needed to ensure $V_{\tau+1}$ is invertable (similar to the first phase in the algorithm of \citet{filippi2010parametric}); the rest of our analysis does not depend on this stochastic assumption.  The same may be achieved by using regularization (see, e.g., \citet{abbasi2011improved}).
%
Another tuning parameter $\alpha$ is used to control the amount of exploration. The larger the $\alpha$ is, the more exploration will be used.

As mentioned earlier, the feature vectors $X_t$ depend on the previous rewards. Consequently, the rewards $\{Y_i, i \in [t]\}$ may not be independent given $\{X_i, i \in [t]\}$.  We instead use results on self-normalized martingales~\cite{abbasi2011improved}, together with a finite-time normality result like \thmref{thm:glm_ucb}, to prove the next theorem.

\begin{theorem} \label{thm:main1}
Fix any $\delta>0$.  There exists a universal constant $C>0$, such that if we run UCB-GLM with $\alpha=\frac{\sigma}{\kappa} \sqrt{\frac{d}{2} \log(1+2T/d) + \log(1/\delta)}$ and $\tau=C\sigma_0^{-2}(d+\log(1/\delta))$, then, with probability at least $1-2\delta$, the regret of the algorithm is upper bounded by
\begin{equation*}
R_T \le \tau + \frac{2 L_{\mu} \sigma d}{\kappa} \log \left( \frac{T}{d\delta}\right) \sqrt{T} \,.
\end{equation*}
\end{theorem}

The theorem shows an $\tilde{O}(d\sqrt{T})$ regret bound that is independent of $K$. Indeed, this rate matches the minimax lower bound up to logarithm factor for the infinite actions contextual bandit problems \cite{dani2008stochastic}.  By choosing $\delta=1/T$ and using the fact that $R_T \le T$, this high-probability result implies a bound on the \emph{expected} regret: $\E [R_T] = \tilde{O}(d\sqrt{T})$.
Our result improves the previous regret bound of \citet{filippi2010parametric} by a $\sqrt{\log T}$ factor. Moreover, the algorithm proposed in \citet{filippi2010parametric} involves a projection step, which is computationally more expensive comparing to UCB-GLM.  Finally, this algorithm works well in practice. We give a heuristic argument for its strong performance in \secref{sec:discussions}, under a specific condition that sometimes are satisfied.

\begin{proof}[Proof of \thmref{thm:main1}]
We first bound the one-step regret.  To do so, fix $t$ and let $X_t^* = x_{t, a_t^*}$ and $\Delta_t=\hat{\theta}_t-\theta^*$, where $a_t^* \in \arg\max_{a\in[K]}\mu(x_{t,a}'\theta^*)$ is an optimal action at round $t$.
The selection of $a_t$ in UCB-GLM implies
\begin{equation} \label{eq:pitstar}
\iprod{X_t^*}{\hat{\theta}_t} + \alpha \norm{X_t^*}_{V_t^{-1}} \le \iprod{X_t}{\hat{\theta}_t} + \alpha \norm{X_t}_{V_t^{-1}}\,.
\end{equation}
Then, we have
\begin{eqnarray*} 
\lefteqn{\iprod{X_t^*}{\theta^*} - \iprod{X_t}{\theta^*} = \iprod{X_t^* - X_t}{\hat{\theta}_t} - \iprod{X_t^* - X_t}{ \hat{\theta}_t - \theta^*}} \\
\label{eq:basic1}&\le&  \alpha \norm{X_t}_{V_t^{-1}} - \alpha \norm{X_t^*}_{V_t^{-1}}- \iprod{X_t^* - X_t}{ \Delta_t} \\
&\le&  \alpha \big ( \norm{X_t}_{V_t^{-1}} - \norm{X_t^*}_{V_t^{-1}} \big) + \norm{X_t^* - X_t}_{V_t^{-1}} \norm{\Delta_t}_{V_t}\,,
\end{eqnarray*}
where the last inequality is due to Cauchy-Schwartz inequality. We have the following two lemmas to bound $\norm{\Delta_t}_{V_t}$ and $\norm{X_t}_{V_t^{-1}}$, respectively.  Their proofs are deferred to the appendix.

\begin{lemma} \label{lm:widthsum}
Let $\{X_t\}_{t=1}^{\infty}$ be a sequence in $\R^d$ satisfying $\norm{X_t} \le 1$.  Define $X_0=\zerovec$ and $V_t=\sum_{s=0}^{t-1} X_s X_s'$. Suppose there is an integer $m$ such that $\lmin(V_{m+1}) \ge 1$, then for all $n>0$,
$$ \sum_{t=m+1}^{m+n} \norm{X_t}_{V_{t}^{-1}} \le \sqrt{2nd \log {\left(\frac{n+m}{d} \right)}}\,. $$
\end{lemma}

\begin{lemma} \label{lm:deltatbound}
Suppose $\lmin(V_{\tau+1}) \ge 1$.  
For any $\delta \in [1/T,1)$, define event
\begin{equation*}
\EvtD \defeq \left\{ \norm{\Delta_t}_{V_t} \le \frac{\sigma}{\kappa} \sqrt{\frac{d}{2} \log(1+2t/d) + \log(1/\delta)} \right\}\,.
\end{equation*}
Then, event $\EvtD$ holds for all $t \ge \tau$ with probability at least $1-\delta$.
\end{lemma}


We now choose $\alpha=\frac{\sigma}{\kappa} \sqrt{\frac{d}{2} \log(1+2T/d) + \log(1/\delta)}$.  If event $\EvtD$ holds for all $t \ge \tau$, then,
\begin{eqnarray*}
\lefteqn{\iprod{X_t^*}{\theta^*} - \iprod{X_t}{\theta^*}} \\
&\le& \alpha \left( \norm{X_t}_{V_t^{-1}} -  \norm{X_t^*}_{V_t^{-1}} + \norm{X_t^*-X_t}_{V_t^{-1}} \right) \\
&\le& 2\alpha \norm{X_t}_{V_t^{-1}}\,.
\end{eqnarray*}
Combining the above with \lemref{lm:widthsum} yields
\begin{align}
\sum_{t=\tau+1}^{T} \big(\iprod{X_t^*}{\theta^*} - \iprod{X_t}{\theta^*}\big) &\le  2\alpha \sqrt{2Td\log \left( \frac{T}{d}\right)} \nonumber \\
&\le
\frac{2d\sigma}{\kappa} \log\left(\frac{T}{d\delta}\right)\sqrt{T}\,.
\label{eqn:ucbglm-phase2-bound}
\end{align}
%

Note that $\mu$ is an increasing Lipschitz function with Lipschitz constant $L_\mu$ and the $\mu$ function is bounded between $0$ and $1$. The regret of algorithm UCB-GLM can be upper bounded as
\begin{eqnarray*}
R_T
&=& \sum_{t=1}^{\tau} \Big( \mu \left(\iprod{X_t^*}{\theta^*}\right) -   \mu \left(\iprod{X_t}{\theta^*}\right) \Big) \\
&& + \sum_{t=\tau+1}^{T} \Big(\mu \left(\iprod{X_t^*}{\theta^*}\right) - \mu \left(\iprod{X_t}{\theta^*} \right)\Big) \\
&\le & \tau  + L_{\mu} \sum_{t=\tau+1}^{T} \Big( \iprod{X_t^*}{\theta^*} - \iprod{X_t}{\theta^*} \Big)\,.
\end{eqnarray*}
The proof can be finished by applying \eqnref{eqn:ucbglm-phase2-bound} and the specified value of $\tau$ to the bound above.
\end{proof}

\subsection{Algorithm SupCB-GLM}

While the algorithm UCB-GLM performs sufficiently well in practice~\cite{liunbiased},  it is unclear whether it can achieve the optimal rates of ${O}(\sqrt{dT\log K})$, when $K$ is fixed and small. As mentioned in \secref{sec:UCB-GLM}, the key technical difficulty in analyzing UCB-GLM is the dependence between samples. Inspired by a technique developed by \citet{auer2003using} to create independent samples for linear contextual bandits, we propose another algorithm SupCB-GLM (\algref{alg: SUPUCB}), which uses algorithm CB-GLM (\algref{alg: CB}) as a sub-routine.

\begin{algorithm}
\caption{CB-GLM} \label{alg: CB} \vspace{0.05in}
\textbf{Input}: parameter $\alpha$, index set $\Psi(t)$, and candidate set $A$. \vspace{-0.1in}
\begin{enumerate}
\item  Let $\hat{\theta}_t$ be the solution of
$$\sum_{i \in \Psi(t)} \left[Y_i - \mu(X_i'\theta)\right]X_i = 0$$
\item $ V_t=\sum_{i \in \Psi(t)} X_i X_i'$
\item \textbf{For} {$a \in A$}, \textbf{do}
 $$ w_{t,a} = \alpha \norm{x_{t,a}}_{V_t^{-1}}, \quad m_{t,a} =  \iprod{x_{t,a}}{\hat{\theta}_t}$$
\textbf{End For}
\end{enumerate}
\end{algorithm}

\begin{algorithm}[h]
\caption{SupCB-GLM} \label{alg: SUPUCB}
\vspace{0.05in} \textbf{Input}: tuning parameter $\alpha$, $\tau$, the number of trials $T$.\vspace{0.05in}

\noindent\textbf{Initialization}: \\ \vspace{0.03in}
\indent\indent ~~~\textbf{for} $t\in[\tau]$, randomly choose $a_t \in [K]$. \\ \vspace{0.0in}
\indent\indent ~~ Set $S=\lfloor \log_2{T} \rfloor$, $\mathcal{F}=\{a1, \cdots, a{\tau}\}$ and $\Psi_0=\Psi_1=\cdots=\Psi_S=\varnothing$. \vspace{0.0in}

\textbf{For} {$t=\tau+1,\tau+2,\cdots,T$} \textbf{do}
 \begin{itemize}
 \item[1.] Initialize $A_1=[K]$ and $s=1$.
 \item[2.] While $a_t=$Null
    \begin{itemize}
    \item[a.] \vspace{-0.05in} Run CB-GLM with $\alpha$ and $\Psi_s \cup \mathcal{F}$ to calculate $m_{t,a}^{(s)}$ and $w_{t,a}^{(s)}$ for all $a \in A_s$.
    \item[b.] If $w_{t,a}^{(s)}>2^{-s}$ for some $a \in A_s$, $$ \textrm{set }  a_t=a \textrm{, ~~update } \Psi_s=\Psi_s \cup \{t\} $$
  \item[c.] Else if $w_{t,a}^{(s)} \le 1/\sqrt{T}$ for all $a \in A_s$, $$\textrm{set } a_t=\argmax_{a \in A_s} m_{t,a}^{(s)} \textrm{, ~~update }\Psi_0=\Psi_0 \cup \{t\} $$ 
  \vspace{-0.1in}  \item[d.] Else if $w_{t,a}^{(s)} \le 2^{-s}$ for all $a \in A_s$,
    $$A_{s+1}=\{a \in A_s, m_{t,a}^{(s)} \ge \max_{j \in A_s} m_{t,j}^{(s)} - 2 \cdot 2^{-s}\},$$
    $\quad s \leftarrow s+1\,.$
    \end{itemize}
 \end{itemize} \vspace{-0.05in}
\textbf{End For} \vspace{0.05in}
\end{algorithm}

This algorithm also relies on the idea of confidence bound to do exploration.  At round $t$, the algorithm screens the candidate actions based on the value of $w_{t,a}^{(s)}$ through $S$ stages until an action is chosen. At stage $s$, we set the confidence level at stage $s$ to be $2^{-s}$. If $w_{t,a}^{(s)}>2^{-s}$ for some $a$, we need to do more exploration on $x_{t,a}$ and thus we choose this action. Otherwise, the actions are filtered in step 2d such that the actions passed to the next stage are close enough to the optimal action. Since all the widths are smaller than $2^{-s}$, if $m_{t,a}^{(s)}<m_{t,j}^{(s)}-2\cdot2^{-s}$ for some $j \in A_s$, the action $a$ can not be the optimal action. The filter process terminates when we have already got accurate estimate of all $x_{t,a}'\theta^*$ up to the $1/\sqrt{T}$ level and we do not need to do exploration. Thus in step 2c we just choose the action that maximizes the estimated mean value. 

Our algorithm is different from the algorithm SupLinRel in \citet{auer2003using} that we directly maximize the mean, rather than the upper confidence bound, in steps~c and d. This modification leads to a simpler algorithm and a cleaner regret analysis. Also, we would like to point out that, unlike SpectralEliminator~\cite{valko2014spectral}, the algorithm can easily handle a changing action set.

The following result, adapted from \citet[lemma 14]{auer2003using}, shows how the algorithm SupCB-GLM will give us independent samples. For the sake of completeness, we also present its proof here. 
\begin{lemma}\label{lm:independence}
For all $s \in [S] $ and $t \in [T]$, given $\{x_{i,a_i}, i \in \Psi_s(t)\}$, the rewards $\{y_{i,a_i}, i \in \Psi_s(t)\}$ are independent random variables.
\end{lemma}
\begin{proof} [Proof of \lemref{lm:independence}]
Since a trial $t$ can only be added to $\Psi_s(t)$ in step 2b of algorithm SupCB-GLM, the event $\{t \in \Psi_s\}$ only depends on the results of trials $\tau \in \cup_{\sigma<s} \Psi_{\sigma}(t)$ and on $w_{t,a}^{(s)}$. From the definition of $w_{t,a}^{(s)}$, we know it only depends on the feature vectors $x_{i,a_i}, i \in \Psi_s(t)$ and on $x_{t,i}$. This implies the lemma.
\end{proof}

With \lemref{lm:independence}, we are able to apply the non-asymptotic normality result \eqnref{eq:normality} and thus to prove our regret bound of Algorithm SupCB-GLM. 

\begin{theorem} \label{thm:main}
For any $0 < \delta < 1$, if we run the SupCB-GLM algorithm with $\tau=\sqrt{dT}$ and $\alpha=\frac{3\sigma}{\kappa} \sqrt{2\log(TK/\delta)}$ for $T \ge T_0$ rounds, where
\begin{equation} \label{eq:condT}
T_0 = \Omega \left( \frac{\sigma^2}{\kappa^4}\max \left \{d^3, \frac{\log(TK/\delta)}{d} \right\} \right),
\end{equation}
the regret of the algorithm is bounded as
\[ R_T \le  45 (\sigma L_{\mu}/\kappa) \sqrt{\log T \log(TK/\delta)\log(T/d)}  \sqrt{dT}\,,\]
with probability at least $1-\delta$. With $\delta = 1/T$, we obtain \[ \E[R_T] = O\left ((\log T)^{1.5} \sqrt{dT \log K} \right). \]
\end{theorem}

The theorem demonstrates an $\tilde{O}(\sqrt{dT\log K})$ regret bound for the algorithm SupCB-GLM. It has been proved in \citet[Theorem 2]{chu2011contextual} that $\sqrt{dT}$ is the minimax lower bound of the expected regret for $K$-armed linear bandits, a special of the GLM bandits considered here. 
Therefore, the regret of our SupCB-GLM algorithm is optimal up to logarithm terms of $T$ and $K$. To the best of our knowledge, this is the first algorithm which achieves the (near-)optimal rate of GLM bandits.


It is worthwhile to compare \thmref{thm:main} with the result in \thmref{thm:main1}. When $K =o(2^{d})$ is small, the rate of SupCB-GLM is faster, and we improve the previous rates by a $\sqrt{d}$ factor. Here, we give a briefly illustration of how we get rid of the extra $\sqrt{d}$ factor. Both in \thmref{thm:main1} and in \citet{filippi2010parametric},  $|x'(\hat{\theta}_n-\theta^*)|$ is upper bounded by using the Cauchy-Schwartz inequality,
\begin{equation} \label{eq:cauchy}
|x'(\hat{\theta}_n-\theta^*)| \le \norm{x}_{V_n^{-1}} \norm{\hat{\theta}_n - \theta^*}_{V_n}\,.
\end{equation}
\lemref{lm:deltatbound} in the supplementary material establishes that
\[ \norm{\hat{\theta}_n - \theta^*}_{V_n} 
\le C_2 \sqrt{d\log(T/\delta)}.\] 
This will lead to an extra $\sqrt{d}$ factor compared to \eqnref{eq:normality}. By using the Cauchy-Schwartz inequality~\eqnref{eq:cauchy}, we only make use of the fact that $\hat{\theta}_n$ is close to $\theta^*$ in the $\ell_2$ sense. However, \eqnref{eq:normality} tells us that actually $\hat{\theta}_n$ is close to $\theta^*$ in \emph{every} direction. This is the reason why we are able to remove the extra $\sqrt{d}$ factor to achieve a near-optimal regret.  It also explains why the bound in \thmref{thm:main1} is tight when $K$ is large. As $K$ goes large, it is likely there is a direction $x$ for which \eqnref{eq:cauchy} is tight.

\begin{proof}[Proof of \thmref{thm:main}]
To facilitate our proof, we first present two technical lemmas.  \lemref{lm:ucb} follows from  \lemref{lm:independence}, \thmref{thm:glm_ucb},  Theorem~5.39 of \citet{vershynin2010introduction} and a union bound.  The proof of \lemref{lm:propSup} is deferred to the appendix.

\begin{lemma} \label{lm:ucb}
Fix $\delta>0$. Choose in SupCB-GLM $\tau=\sqrt{dT}$ and $\alpha=\frac{3\sigma}{\kappa} \sqrt{2\log(TK/\delta)}$.  Suppose $T$ satisfies condition \eqnref{eq:condT}.  Define the following event:
\begin{eqnarray}
\EvtX &\defeq& \{|m_{t,a}^{(s)} - x_{t,a}'\theta^*| \le w_{t,a}^{(s)}, 
\label{eq:ucb} \\
&& \forall t\in[\tau+1,T], s\in[S], a\in[K]\} \nonumber
\end{eqnarray}
Then, event $\EvtX$ holds with probability at least $1-\delta $.
\end{lemma}

\begin{proof} [Proof of \lemref{lm:ucb}]
By \lemref{lm:independence}, we have independent samples now. Then to apply \thmref{thm:glm_ucb}, the key is to lower bound the minimum eigenvalue of $V_t$. Note that we randomly select the feature vectors at the first $\tau=\sqrt{dT}$ rounds, that is, they are independent. Moreover, the feature vectors are bounded. Thus, $X_1, X_2, \ldots, X_\tau$ are independent sub-Gaussian with parameter $1$.  It follows from Proposition \ref{pro:spectral} that
\[
\lmin(V_t) \ge \lmin(V_\tau) \ge c\sqrt{dT}
\]
for some constant $c$ with probability at least $1-\exp(-\sqrt{dT})$. By Theorem \ref{thm:glm_ucb} and union bound, we have the desired result under condition  \eqnref{eq:condT}.
\end{proof}

\begin{lemma} \label{lm:propSup}
Suppose that event $\EvtX$ holds, and that in round $t$, the action $a_t$ is chosen at stage $s_t$.  Then, $a_t^{*} \in A_s$ for all $s \le s_t$. Furthermore, we have
\begin{align*}
&\mu(x_{t,a_t^*}' \theta^*) - \mu(x_{t,a_t}' \theta^*) \\
&\le \left\{ \begin{array}{ll}
 (8 L_\mu)/2^{s_t} & \textrm{if $a_t$ is selected in step 2b}\\
 (2L_\mu)/\sqrt{T} & \textrm{if $a_t$ is selected in step 2c}\,.
  \end{array} \right.
\end{align*}
\end{lemma}

Define $V_{s,t}=\sum_{t \in \Psi_s(T)} X_tX_t'$, then by \lemref{lm:widthsum}, 
\begin{eqnarray*}
\sum_{t \in \Psi_s(T)} w_{t,a_t}^{(s)} &=& \sum_{t \in \Psi_s(T)} \alpha(\delta) \|x_{t,a_t}\|_{V_{s,t}^{-1}} \\
&\le& \alpha(\delta) \sqrt{2 d \log(T/d) |\Psi_s(n)|}\,.
\end{eqnarray*}
On the other hand, by the step 2b of SupCB-GLM,
\begin{equation*}
\sum_{t \in \Psi_s(T)} w_{t,a_t}^{(s)} \ge 2^{-s} |\Psi_s(T)|.
\end{equation*}
Combining the above two inequalities gives us
\begin{equation} \label{eq:cardbound}
|\Psi_s(T)| \le 2^s \alpha(\delta) \sqrt{2 d\log{(T/d)} |\Psi_s(T)| }.
\end{equation}
Let $\Psi_0$ be the collection of trials such that $a_t$ is chosen in step 2c. Since we have chose $S=\log_2 T$, each $t \in [\tau+1, T]$ must be in one of $\Psi_s$ and hence, $\{\tau,\tau+1,\ldots,T\}=\Psi_0\cup \left(\cup_{s=1}^{S}\Psi_s(T)\right)$. If we set $\tau=\sqrt{dT}$, we have
\begin{eqnarray*}
\lefteqn{R_T = \sum_{t=1}^{\tau} \left( \mu(x_{t,a_t^*}' \theta^*) - \mu(x_{t,a_t}' \theta^*) \right)} \\
&& + \sum_{t=\tau+1}^{T} \left( \mu(x_{t,a_t^*}' \theta^*) - \mu(x_{t,a_t}' \theta^*) \right) \\
&\le & \tau + \sum_{t \in \Psi_0}  \left( \mu(x_{t,a_t^*}' \theta^*) - \mu(x_{t,a_t}' \theta^*)\right) \\
&& + \sum_{s=1}^{S}  \sum_{t \in \Psi_s(T)} \left( \mu(x_{t,a_t^*}' \theta^*) - \mu(x_{t,a_t}' \theta^*)\right) \\
&\le& \sqrt{dT} + T \cdot  \frac{2L_{\mu}}{\sqrt{T}} + \sum_{s=1}^{S}  L_\mu \cdot 2^{3-s} \cdot |\Psi_s(T)| \\
&\le& \sqrt{dT} + 2L_\mu \sqrt{T} \\
& & + 8L_\mu \alpha(\delta)  \sum_{s=1}^{S} \sqrt{2d \log\frac{T}{d} |\Psi_s(T)| } \\
&\le& \sqrt{dT} + 2L_\mu \sqrt{T} + 8L_\mu \alpha(\delta)   \sqrt{2d \log(T/d)} \sqrt{ST} \\
&\le& 45 (\sigma L_{\mu}/\kappa) \sqrt{\log T \log(TK/\delta)\log(T/d)} \sqrt{dT}\,,
\end{eqnarray*}
with probability at least $1 - \delta$. Here, the first inequality is due to the assumption that $0 \le \mu \le 1$. The second inequality is \lemref{lm:propSup}. The third inequality is the inequality \eqnref{eq:cardbound} and the fourth inequality is implied by Cauchy-Schwartz. This completes the proof of the high-probability result. 
\end{proof}

\section{Discussions}
\label{sec:discussions}

In this paper, we propose two algorithms for $K$-armed bandits with generalized linear models. While the first algorithm, UCB-GLM, achieves the optimal rate for the case of infinite number of arms, the second algorithm SupCB-GLM is provable optimal for the case of finite number actions at each round. However, it remains open whether UCB-GLM can achieve the optimal rate for small $K$.

\subsection{A better regret bound for UCB-GLM}

A key quantity in determining the regret of UCB-GLM is the minimum eigenvalue of $V_t$. If we make an addition assumption on the minimum eigenvalue of $V_t$, we will be able to prove an $O(\sqrt{dT})$ regret bound for UCB-GLM.

\begin{theorem} \label{thm:maindis}
We run algorithm UCB-GLM with $\tau=\frac{8\sigma^2}{\kappa^2}d\log T$ and $\alpha \le L_{\mu} \sigma/\kappa$. For any $\delta \in [1/T,1)$, suppose there is an universal constant $c$ such that
\begin{equation} \label{eq:assump4}
\sum_{t=\tau+1}^{T} \lmin^{-1/2}(V_t) \le  c \sqrt{T}.
\end{equation}
holds with probability at least $1-\delta$, and 
\begin{equation} \label{eq:condition1}
T=\Omega \left(\frac{\sigma R}{ \kappa L_{\mu} } d \log^2 T \right).
\end{equation}
Then, the regret of the algorithm is bounded by
\begin{equation*}
R_T \le \frac{C L_{\mu} \sigma}{\kappa} \sqrt{dT\log(T/\delta)}
\end{equation*}
with probability at least $1-2\delta$, where $C$ is a positive, universal constant. \end{theorem}

This theorem provides some insights of why UCB-GLM performs well in practice.  Although the condition in \eqnref{eq:assump4} is hard to check and may be violated in some cases, for example, in $K$-armed bandits, we provide a heuristic argument to justify this assumption in a range of problems. When $t$ is large enough, our estimator $\hat{\theta}_t$ is very close to $\theta^*$. If we assume there is a positive gap between $\iprod{x_{t,a_t^*}}{\theta^*}$ and $\iprod{x_{t,a}}{\theta^*}$ for all $a \neq a_t^*$, we will have $a_t=a_t^*$ after, for example, $\sqrt{T}$ steps. Since $\left \{x_{t,a},a \in [K] \right \}$ are independent for $t \in [T]$, $\{x_{t,a_t^*}\}$ are also independent samples. Then $V_t/t$ will be well-approximated by the covariance matrix of $x_{t,a_t^*}$, which we denote by $\Sigma_0$.  In many problem in practice, especially when features are \emph{dense}, it is unlikely the feature vector $x_{t,a_t^*}$ lies in a low-dimensional subspace of $\R^d$.  It implies that $\Sigma_0$ has full rank, and that we will have $\lmin(V_t)=\Theta(t \cdot \lmin(\Sigma_0))$ when $t$ is large enough.  It follows that,
\begin{equation*} 
\sum_{t=\tau+1}^{T} \frac{1}{\sqrt{\lmin(V_t)}} = \sum_{t=\tau+1}^{T} \Theta(t^{-1/2}) = O(\sqrt{T}).
\end{equation*}

It should be cautioned that, since we do not know the distribution of our feature vectors, we cannot assume the above gap exists.  It is therefore challenging to make the above arguments rigorous.
In fact, when studying the ARIMA model in time series, \citet[Example 1]{lai1982least} provide an example such that $\lmin(V_t)=O(\log t)$.

\subsection{Open Questions}

\paragraph{Computational efficient algorithms.}

While UCB-GLM and SupCB-GLM enjoy good theoretical properties, they can be expensive in some applications.  First, they require inverting a $d\times d$ matrix in every step, a costly operation when $d$ is large.  Second, at step $t$, the MLE is computed using $\Theta(t)$ samples, meaning that the per-step complexity grows at least linearly with $t$ for a straightforward implementation of the algorithms.  It is therefore interesting to investigate more scalable alternatives.  It is possible to use a first-order, iterative optimization procedure to amortize the cost, analogous to the approach of \citet{agarwal2014taming}.
%

\paragraph{$K$-dependent lower bound.}
Currently, all the lower bound results on (generalized) linear bandits have no dependence on $K$, the number of arms. The minimax lower bound will be of particularly interest because all current lower bound results assume that $K \le d$. Although it will at most be a logarithm dependence on $K$, it is still a theoretically interesting question.

\paragraph{Randomized algorithms with optimal regret rate.}
As opposed to the deterministic, UCB-style algorithms studied in this paper, randomized algorithms like EXP4~\cite{auer2002nonstochastic} and Thompson Sampling~\cite{Thompson33Likelihood} have advantages in certain situations, for example, when reward observations are delayed~\cite{Chapelle12Empirical}.  Recently developed techniques for analyzing Bayes regret in BLM bandits~\cite{Russo14Learning} may be useful to analyze the cumulative regret considered here.

\clearpage


\bibliography{bandit}
\bibliographystyle{icml2017}


\clearpage

\onecolumn

\appendix

\begin{center}
\large
Supplementary for \\
\Large
Provably Optimal Algorithms for Generalized Linear Contextual Bandits
\vspace{10mm}
\end{center}

\section{Proof of \thmref{thm:glm_ucb}}
\label{sec:glm_ucb_proof}

In the following, for simplicity, we will drop the subscript $n$ when there is no ambiguity.  Therefore, $V_n$ is denoted $V$ and so on.

To prove normality-type results of the maximum likelihood estimator $\hat{\theta}$, typically we first show the $n^{-1/2}$-consistency of $\hat{\theta}$ to $\theta^*$. Then, by using a second-order Taylor expansion or Newton-step, we can prove the desired normality of $\hat{\theta}$.
More details can be found in standard textbooks such as \citet{van2000asymptotic}. 

Since $m$ is twice differentiable with $\ddot{m} \ge 0$, the maximum-likelihood estimation can be written as the solution to the following equation
\begin{equation} \label{eq:score}
\sum_{i=1}^{n} \left(Y_i-\mu(X_i'\theta)\right) X_i= 0\,.
\end{equation}
Define 
$G(\theta) \defeq \sum_{i=1}^{n} \left(\mu(X_i'\theta)-\mu(X_i'\theta^*)\right) X_i$,
and we have 
\begin{equation} \label{eq:mle}
G(\theta^*)=0 \;\; \text{and }\; G(\hat{\theta})=\sum_{i=1}^{n} \epsilon_i X_i\,,
\end{equation} 
where the noise $\epsilon_i$ is defined in \eqnref{eq:model1}.  For convenience, define $Z \defeq G(\hat{\theta})=\sum_{i=1}^{n} \epsilon_i X_i$.

\paragraph{Step 1: Consistency of $\hat{\theta}$.} We first prove the consistency of $\hat{\theta}.$ For any $\theta_1, \theta_2 \in \R^d$, mean value theorem implies that there exists some $\bar{\theta}=v\theta_1+(1-v)\theta_2$ with $0<v<1$, such that
\begin{equation} \label{eq:taylor1}
G(\theta_1) - G(\theta_2) =\left [\sum_{i=1}^{n}\dot{\mu}(X_i' \bar{\theta})X_iX_i' \right](\theta_1-\theta_2):=F(\bar{\theta})(\theta_1-\theta_2)
\end{equation}
Since $\dot{\mu}>0$ and $\lmin(V) > 0$, we have
\[
(\theta_1-\theta_2)'(G(\theta_1)-G(\theta_2)) \ge (\theta_1-\theta_2)' (\kappa V) (\theta_1-\theta_2) > 0
\]
for any $\theta_1 \neq \theta_2$. Hence, $G(\theta)$ is an injection from $\R^d$ to $\R^d$, and so $G^{-1}$ is a well-defined function. Consequently, \eqnref{eq:score} has a unique solution $ \hat{\theta} = G^{-1}(Z)$.

Let us consider an $\eta$-neighborhood of $\theta^*$, $\B_\eta \defeq \{\theta ~:~ \norm{\theta-\theta^*} \le \eta\}$, where $\eta>0$ is a constant that will be specified later. Note that $\mathcal{B}_{\eta}$ is a convex set, thus $\bar{\theta} \in \mathcal{B}_{\eta}$ as long as $\theta_1,\theta_2\in\B_\eta$. Define $\kappa_{\eta} \defeq \inf_{\theta \in \B_\eta} \dot{\mu}(x'\theta)>0$. From \eqnref{eq:taylor1},  for any $\theta \in \mathcal{B}_\eta$,
\begin{eqnarray*}
\norm{G(\theta)}_{V^{-1}}^2 &=& \norm{G(\theta)-G(\theta^*)}_{V^{-1}}^2 \\
&=& (\theta-\theta^*)' F(\bar{\theta}) V^{-1} F(\bar{\theta}) (\theta-\theta^*) \\
&\ge& \kappa_{\eta}^2 \lmin(V) \norm{\theta - \theta^*}^2,
\end{eqnarray*}
where the last inequality is due to the fact that $F(\bar{\theta}) \mge \kappa_\eta V$.

On the other hand, Lemma~A of \citet{chen1999strong} implies that
\[
\left \{\theta ~:~\norm{G(\theta)}_{V^{-1}} \le  \kappa_{\eta} \eta \sqrt{\lmin(V)} \right \}  \subset \mathcal{B}_{\eta}\,.
\]
Now it remains to upper bound $\norm{Z}_{V^{-1}}=\norm{G(\hat{\theta})}_{V^{-1}}$ to ensure $\hat{\theta}\in\B_\eta$.  
To do so, we need the following technical lemma, whose proof is deferred to \secref{sec:lemmas}. 
\begin{lemma} \label{lm:hoeffding}
Recall $\sigma$ which is the constant in \eqnref{eq:glm2}.  For any $\delta>0$, define the following event:
\begin{equation*}
\EvtG \defeq \left\{\norm{Z}_{V^{-1}} \le  4\sigma \sqrt{d+\log(1/\delta)}\right\}\,.
\end{equation*}
Then, $\EvtG$ holds with probability at least $1-\delta$.
\end{lemma}

Suppose $\EvtG$ holds for the rest of the proof.  Then, $\eta \ge \frac{4\sigma}{\kappa_{\eta}} \sqrt{\frac{d + \log(1/\delta)}{\lmin(V)}}$ implies $\norm{\hat{\theta}_t-\theta} \le \eta$. Since $\kappa=\kappa_1$, we have $\kappa_\eta \ge \kappa$ as long as $\eta \le 1$. Thus, we have
\begin{equation} \label{eq:intial} 
\norm{\hat{\theta}-\theta} \le \frac{4\sigma}{\kappa} \sqrt{\frac{d  + \log(1/\delta)}{\lmin(V)}} \le 1\,,
\end{equation}
when $\lmin (V) \ge 16\sigma^2\left[d + \log(1/\delta)\right]/\kappa^2$. 

\paragraph{Step 2: Normality of $\hat{\theta}$.} Now, we are ready to precede to prove the normality result. 
The following assumes $\EvtG$ holds (which is high-probability event, according to \lemref{lm:hoeffding}).

Define $\Delta \defeq \hat{\theta}-\theta^*$. It follows from \eqnref{eq:taylor1} that there exists a $v \in [0,1]$ such that
\[
Z=G(\hat{\theta})-G(\theta^*)=(H+E)\Delta\,,
\] 
where $\tilde{\theta} \defeq v \theta^*+(1-v)\hat{\theta}$, $H \defeq F(\theta^*)=\sum_{i=1}^{n}\dot{\mu}(X_i' \theta^*)X_iX_i' $ and $E \defeq F(\tilde{\theta})-F(\theta^*)$. Intuitively, when $\hat{\theta}$ and $\theta^*$ are close, elements in $E$ are small.  By the mean value theorem, 
\begin{eqnarray*}
E =\sum_{i=1}^{n} \left(\dot{\mu}(X_i'\tilde{\theta}) - \dot{\mu}(X_i'{\theta^*})\right)X_iX_i'= \sum_{i=1}^{n} \ddot{\mu}(r_i)X_i'\Delta X_iX_i'
\end{eqnarray*}
for some $r_i \in \R$. Since $\ddot{\mu} \le M_\mu$ and $v \in [0,1]$, for any $x \in \R^d \setminus \{\zerovec\}$, we have
\begin{eqnarray*}
x' H^{-1/2}EH^{-1/2} x &=& (1-v) \sum_{i=1}^{n} \ddot{\mu}(r_i) X_i' \Delta \norm{x' H^{-1/2} X_i}^2 \\
&\le& \sum_{i=1}^{n} M_\mu \norm{X_i} \norm{\Delta} \norm{x' H^{-1/2} X_i}^2 \\
&\le& M_\mu \norm{\Delta}  \left( x' H^{-1/2} \left(\sum_{i=1}^{n} X_i X_i' \right) H^{-1/2} x \right) \\ 
&\le& \frac{M_\mu }{\kappa} \norm{\Delta} \norm{x}^2\,,
\end{eqnarray*}
where we have used the assumption that $\norm{X_i} \le 1$ for the second inequality. Therefore,
\begin{equation}
\norm{H^{-1/2}EH^{-1/2}} \le \frac{M_\mu}{\kappa} \norm{\Delta} \le  \frac{4M_\mu \sigma}{\kappa^2} \sqrt{\frac{d + \log(1/\delta)}{\lmin(V)}}\,. \label{eqn:normality}
\end{equation}
When $\lmin(V) \ge 64 M_\mu^2 \sigma^2(d + \log(1/\delta))/\kappa^4$, we have
\begin{equation}
\norm{H^{-1/2}EH^{-1/2}} \le 1/2\,.
\label{eqn:half}
\end{equation}

Now we are ready to prove the theorem. For any $x \in \R^d$, 
\begin{equation} 
x' (\hat{\theta}-\theta^*) \,=\, x' (H+E)^\mi Z \,=\, x' H^\mi Z - x' H^\mi E (H+E)^\mi Z\,.
\label{eqn:decomposition}
\end{equation}
Note that the matrix $(H+E)$ is nonsingular, so its inversion exists.

For the first term, $\{\epsilon_i\}$ are sub-Gaussian random variables with sub-Gaussian parameter $\sigma$.  Define
\[
D \defeq \left[X_1, X_2, \ldots, X_n \right]' \in \R^{n \times d}
\]
to be the design matrix.  Hoeffding inequality gives
\begin{equation}
\Prob\{|x' H^{-1} Z| \ge t\} \le 2 \exp\left\{-\frac{t^2}{2\sigma^2 \norm{x' H^{-1} D'}^2}\right\}\,.
\label{eqn:first-term-hoeffding}
\end{equation}
Since $H \mge \kappa V = \kappa D' D$, we have 
\[
\norm{x' H^{-1} D'}^2 = x' H^{-1} D' D H^{-1} x \le \frac{1}{\kappa^2} x' V^{-1} x = \frac{1}{\kappa^2} \norm{x}_{V^\mi}^2\,,
\]
so \eqnref{eqn:first-term-hoeffding} implies
\[
\Prob\{|x' H^{-1} Z| \ge t\} \le 2 \exp\left\{-\frac{t^2 \kappa^2}{2\sigma^2 \norm{x}_{V^\mi}^2}\right\}\,.
\]
Let the right-hand side be $2\delta$ and solve for $t$, we obtain that with probability at least $1-2\delta$,
\begin{equation}
|x' H^{-1} Z| \le \frac{\sqrt{2}\sigma}{\kappa} \sqrt{\log(1/\delta)} \norm{x}_{V^{-1}}\,.
\label{eqn:first-term-bound}
\end{equation}

For the second term, 
\begin{eqnarray}
 |x' H^{-1} E (H+E)^{-1} Z|  \nonumber &\le& \norm{x} _{H^{-1}} \norm{H^{-1/2}E (H+E)^{-1} Z} \\
 \nonumber &\le& \norm{x}_{H^{-1}} \norm{H^{-1/2}E (H+E)^{-1}H^{1/2}} \norm{Z}_{H^{-1}} \\
 \label{eq:term2} &\le& \frac{1}{\kappa} \norm{x}_{V^{-1}} \norm{H^{-1/2}E (H+E)^{-1}H^{1/2}} \norm{Z}_{V^{-1}}\,,
\end{eqnarray}
where the last inequality is due to the fact that $H \mge \kappa V$. Since $(H+E)^{-1}=H^{-1}-H^{-1}E(H+E)^{-1}$, we have
\begin{eqnarray*}
\norm{H^{-1/2}E (H+E)^{-1}H^{1/2}} &=& \norm{H^{-1/2}E \left(H^{-1} - H^{-1} E (H+E)^{-1}\right)H^{1/2}} \\
&=& \norm{H^{-1/2} E H^{-1/2} + H^{-1/2} E H^{-1} E (H+E)^{-1} H^{1/2}} \\
&\le& \norm{H^{-1/2} E H^{-1/2}} + \norm{H^{-1/2} E H^{-1/2}} \norm{H^{-1/2}E (H+E)^{-1}H^{1/2}}\,.
\end{eqnarray*}
By solving this inequality, we get
\begin{equation*} \label{eq:matrixinv}
\norm{H^{-1/2}E (H+E)^{-1}H^{1/2}} \le \frac{\norm{H^{-1/2} E H^{-1/2}}}{1-\norm{H^{-1/2} E H^{-1/2}}} \le 2 \norm{H^{-1/2} E H^{-1/2}} \le \frac{8 M_\mu \sigma}{\kappa^2} \sqrt{\frac{d + \log(1/\delta)}{\lmin(V)}}\,,
\end{equation*}
where we have used \eqnref{eqn:half} and \eqnref{eqn:normality} in the second and third inequalities, respectively.
Combining it with \eqnref{eq:term2} and the bound in $\EvtG$, we have
\begin{equation} 
|x' H^{-1} E (H+E)^{-1} Z| \,\le\, \frac{32 M_\mu \sigma^2}{\kappa^3} \frac{d + \log(1/\delta)}{\sqrt{\lmin(V)}} \norm{x}_{V^{-1}}.
\label{eqn:second-term-bound}
\end{equation}
From \eqnref{eqn:decomposition}, \eqnref{eqn:first-term-bound} and \eqnref{eqn:second-term-bound}, one can see that \eqnref{eq:normality} holds as long as the lower bound \eqnref{eq:lmin1} for $\lmin(V)$ holds.  Finally, an application of a union bound on two small-probability events (given in \lemref{lm:hoeffding} and \eqnref{eqn:first-term-bound}, respectively) asserts that \eqnref{eq:normality} holds with probability at least $1-3\delta$.

\section{Proof of \proref{pro:spectral}}

In the following, for simplicity, we will drop the subscript $n$ when there is no ambiguity.  Therefore, $V_n$ is denoted $V$ and so on.

Let $X$ be a random vector drawn from the distribution $\nu$.  Define $Z\defeq \Sigma^{-1/2}X$.  Then $Z$ is isotropic, namely, $\E[ZZ'] = \Id_d$.  Define $U=\sum_{t=1}^n Z_tZ_t' = \Sigma^{-1/2}V\Sigma^{-1/2}$.
From \lemref{lem:rmsv}, we have that, for any $t$, with probability at least $1-2\exp(-C_2t^2)$,
\[
\lmin(U)\ge n - C_1 \sigma^2 \sqrt{nd} - \sigma^2 t \sqrt{n} \,. 
\]
where $\sigma$ is the sub-Gaussian parameter of $Z$, which is upper-bounded by $\norm{\Sigma^{-1/2}} = \lmin^{-1/2}(\Sigma)$ (see, e.g., \citet{vershynin2010introduction}).  
We thus can rewrite the above inequality (which holds with probability $1-\delta$ as
\[
\lmin(U) \ge n - \lmin^{-1}(\Sigma) \left(C_1 \sigma^2 \sqrt{nd} + t \sqrt{n} \right) \,.
\]

We now bound the minimum eigenvalue of $V$, as follows:
\begin{eqnarray*}
\lmin(V) &=& \min_{x\in\B^d} x'Vx \\
&=& \min_{x\in\B^d}x'\Sigma^{1/2}U\Sigma^{1/2}x \\
&\ge& \lmin(U) \min_{x\in\B^d}x'\Sigma x \\
&=& \lmin(U) \lmin(\Sigma) \\
&\ge& \lmin(\Sigma) \left( n - \lmin^{-1}(\Sigma) (C_1 \sigma^2 \sqrt{nd} + t \sqrt{n} ) \right) \\
&=& \lmin(\Sigma) n - C_1 \sqrt{nd} - C_2 \sqrt{n \log(1/\delta)}\,.
\end{eqnarray*}

Finally, it can be verified (\lemref{lem:quad-ineq}) that the last expression above is no less than $B$ as long as
\[
n \ge \left(\frac{C_1 \sqrt{d} + C_2 \sqrt{\log(1/\delta)}}{\lmin(\Sigma)}\right)^2 + \frac{2B}{\lmin(\Sigma)} \,,
\]
finishing the proof.


\section{Technical Lemmas and Proofs} \label{sec:lemmas}

\subsection{Proof of \lemref{lm:hoeffding}}

Noting that 
$$\|Z\|_{V^{-1}}=\|V^{-1/2}Z\|_2 = \sup_{\|a\|_2 \le 1}\iprod{a}{V^{-1/2}Z},$$
let $\hat{\B}$ be a $1/2$-net of the unit ball $\B^d$. Then $|\hat{\B}| \le 6^d$~\citep[Lemma 4.1]{pollard1990empirical}, and for any $x \in \B^d$, there is a $\hat{x} \in \hat{\B}$ such that $\norm{x-\hat{x}} \le 1/2$. Consequently, 
\begin{eqnarray*} 
\iprod{x}{V^{-1/2}Z} &=& \iprod{\hat{x}}{V^{-1/2}Z} + \iprod{x-\hat{x}}{V^{-1/2}Z} \\
&=& \iprod{\hat{x}}{V^{-1/2}Z} +  \norm{x-\hat{x}} \iprod{\frac{x-\hat{x}}{\norm{x-\hat{x}}}}{V^{-1/2}Z} \\
&\le& \iprod{\hat{x}}{V^{-1/2}Z} + \frac{1}{2} \sup_{z \in \B^d} \iprod{z}{V^{-1/2}Z}.
\end{eqnarray*}
Taking supremum on both sides, we get
$$ \sup_{x \in \B^d} \iprod{x}{V^{-1/2}Z} \le 2 \max_{\hat{x} \in \hat{\B}} \iprod{\hat{x}}{V^{-1/2}Z}\,.$$
Then a union bound argument implies
\begin{eqnarray*}
\mathbb{P}\left\{\|Z\|_{V^{-1}} > t\right\} &\le& \mathbb{P}\left \{ \max_{\hat{x} \in \hat{\B}} \iprod{\hat{x}}{V^{-1/2}Z} > t/2\right\} \\
&\le& \sum_{\hat{x} \in \hat{\B}} \mathbb{P}\left\{ \iprod{\hat{x}}{V^{-1/2}Z} > t/2 \right \} \\
&\le& \sum_{\hat{x} \in \hat{\B}} \exp \left\{ - \frac{t^2}{8\sigma^2 \norm{\hat{x}'V^{-1/2} X'}^2} \right\} \\
&\le& \exp \left\{ - t^2/(8\sigma^2) +  d \log 6 \right\},
\end{eqnarray*}
where we have used Hoeffding's inequality for the third inequality and $|\hat{\B}| \le 6^d$ for the last inequality. A choice of $t=4\sigma \sqrt{d+\log(1/\delta)}$ completes the proof.

\subsection{Proof of \lemref{lm:widthsum}}

By \citet[Lemma~11]{abbasi2011improved}, we have
$$ \sum_{t=m+1}^{m+n} \norm{X_t}_{V_{t}^{-1}}^2 
\le 2 \log \frac{\det V_{m+n+1}}{\det V_{m+1}} 
\le 2 d \log \left( \frac{\tr{V_{m+1}} + n}{d} \right) - 2 \log \det V_{m+1}\,. $$
Note that $\tr{V_{m+1}} = \sum_{t=1}^{m} \tr{X_tX_t'} = \sum_{t=1}^{m} \|X_t\|^2 \le m$ and that $\det V_{m+1} = \prod_{i=1}^{d} \lambda _i \ge \lmin^d(V_{m+1}) \ge 1$, where $\{\lambda_i\}$ are the eigenvalues of $V_{m+1}$. Applying Cauchy-Schwartz inequality yields
$$ \sum_{t=m+1}^{m+n} \norm{X_t}_{V_{t}^{-1}} \le \sqrt{n \sum_{t=m+1}^{m+n} \norm{X_t}_{V_{t}^{-1}}^2} \le \sqrt{2nd \log {\left(\frac{n+m}{d} \right)}}\,. $$

\subsection{Proof of \lemref{lm:deltatbound}}

Define $G_t(\theta) = \sum_{i=1}^{t-1}(\mu(X_i'\theta)-\mu(X_i'\theta^*))X_i$ and $Z_t=\sum_{i=1}^{t-1}\epsilon_i X_i$. Following the same argument as in the proof of \thmref{thm:glm_ucb}, we have $G_t(\hat{\theta}_t)=Z_t$ and
\begin{equation} \label{eq:aaa}
\|G_t(\theta)\|_{V_t^{-1}}^2 \ge \kappa^2 \|\theta-\theta^*\|_{V_t}^2 
\end{equation}
for any $\theta \in \{\theta~:~\norm{\theta-\theta^*} \le 1\}$.  Combining \eqnref{eq:aaa} with the following lemma and the equality $Z_t=G_t(\hat{\theta}_t)$ completes the proof.

\begin{lemma} \label{lm:self}
Suppose there is an integer $m$ such that $\lmin(V_m) \ge 1$, then for any $\delta \in (0,1)$, with probability at least $1-\delta$, for all $t>m$,
\[
\norm{Z_t}_{V_t^{-1}}^2 \le 4 \sigma^2 \left(\frac{d}{2} \log(1+2t/d) + \log(1/\delta) \right) \,. \]
\end{lemma}

\begin{proof}
For convenience, fix $t$ such that $t>m$, and denote $V_t$ and $Z_t$ by $V$ and $Z$, respectively.  Furthermore, define $\bar{V} \defeq V+\lambda I$ and let $\onevec$ be the vector of all $1$s.  It is easy to observe that
\begin{equation}
\norm{Z}_{V^{-1}}^2 = \norm{Z}_{\bar{V}^{-1}}^2 + Z'(V^{-1} - \bar{V}^{-1})Z\,.
\label{eqn:znorm-decomposition}
\end{equation}
We start with bounding the second term.  The Sherman–Morrison formula gives
\[
\bar{V}^{-1} = V^{-1} - \frac{\lambda V^{-2}}{1+\lambda \onevec'V^{-1}\onevec}\,.
\]
Since $\onevec'V^{-1}\onevec\ge0$, the above implies that
\begin{eqnarray*}
0 &\le& Z'(V^{-1} - \bar{V}^{-1})Z \\
&\le& \lambda Z' V^{-2} Z \\
&\le& \lambda \norm{V^{-1}} \norm{Z}_{V^{-1}}^2 \\
&=& \frac{\lambda}{\lmin(V)}\norm{Z}_{V^{-1}}^2\,.
\end{eqnarray*}
Since $\lmin(V) \ge \lmin(V_m) \ge 1$, we now have
\[
0 \le Z'(V^{-1} - \bar{V}^{-1})Z \le \lambda \norm{Z}_{V^{-1}}^2\,.
\]
The above inequality together with \eqnref{eqn:znorm-decomposition} implies that
\[
\norm{Z}_{V^{-1}}^2 \le (1-\lambda)^{-1} \norm{Z}_{\bar{V}^{-1}}^2\,.
\]
The proof can be finished by applying Theorem~1 and Lemma~10 from \citet{abbasi2011improved} to bound $\norm{Z}_{\bar{V}^{-1}}^2$, using $\lambda=1/2$.
\end{proof} 

\subsection{Proof of \lemref{lm:propSup}}

We will prove the first part of the lemma by induction. It is easy to check the lemma holds for $s=1$. Suppose we have $a_t^* \in A_s$ and we want to prove $a_t^* \in A_{s+1}$. Since the algorithm proceeds to stage $s+1$, we know from step 2b that
$$|m_{t,a}^{(s)} - x_{t,a}'\theta^*| \le w_{t,a}^{(s)} \le 2^{-s}$$
for all $a \in A_{s}$. Specially, it holds for $a=a_t^*$ because $a_t^* \in A_{s}$ by our induction step. Then the optimality of $a_t^*$ implies
$$ m_{t,a_t^*}^{(s)} \ge  x_{t,a_t^*}'\theta^* - 2^{-s} \ge x_{t,a}'\theta^* - 2^{-s} \ge m_{t,a}^{(s)} - 2\cdot2^{-s}$$
for all $a \in A_s$. Thus we have $a_t^* \in A_{s+1}$ according to step 2d.

Suppose $a_t$ is selected at stage $s_t$ in step 2b. If $s_t=1$, obviously the lemma holds because $0 \le \mu(x) \le 1$ for all $x$. If $s_t>1$, since we have proved $a_t^* \in A_{s_t}$, again step 2b at stage $s_t-1$ implies 
$$|m_{t,a}^{(s_t-1)} - x_{t,a}'\theta^*| \le 2^{-s_t+1}$$
for $a=a_t$ and $a=a_t^*$. Step 2d at stage $s_t-1$ implies
$$ m_{t,a_t^*}^{(s_t-1)} -  m_{t,a_t}^{(s_t-1)} \le 2\cdot 2^{-s_t+1}\,.$$
Combining above two inequalities, we get 
$$ x_{t,a_t}'\theta^* \ge m_{t,a_t}^{(s_t-1)} - 2^{-s_t+1} 
\ge m_{t,a_t^*}^{(s_t-1)} - 3\cdot 2^{-s_t+1} \ge x_{t,a_t^*}'\theta^*- 4\cdot 2^{-s_t+1}\,.$$
When $a_t$ is selected in step 2c, since $m_{t,a_t}^{(s_t)} \ge m_{t,a_t^*}^{(s_t)}$, we have
$$ x_{t,a_t}'\theta^* \ge m_{t,a_t}^{(s_t)} - 1/\sqrt{T} 
\ge m_{t,a_t^*}^{(s_t)} - 1/\sqrt{T} \ge x_{t,a_t^*}'\theta^*- 2/\sqrt{T}\,.$$
Using the fact that $\mu(x_1)-\mu(x_2) \le L_\mu (x_1-x_2)$ for $x_1 \ge  x_2$, we will get the desired result.

\subsection{Proof of \lemref{lem:quad-ineq}}

\begin{lemma}
\label{lem:quad-ineq}
Let $a$ and $b$ be two positive constants.  If $m \ge a^2 + 2b$, then $m - a\sqrt{m} - b \ge 0$.
\end{lemma}

\begin{proof}
The function $t \mapsto t^2-at-b$ is monotonically increasing for $t \ge a/2$.  Since $m \ge a^2 + 2b$, we have $\sqrt{m} \ge a/2$, so
\begin{eqnarray*}
m - a\sqrt{m} - b &\ge& a^2 + 2b - a \sqrt{a^2 + 2b} - b \\
&\ge& a^2 + b - a \sqrt{a^2 + 2b + b^2/a^2} \\
&=& a^2 + b - a \sqrt{(a + b/a)^2} \\
&=& a^2 + b - a (a + b/a) \\
&=& 0\,.
\end{eqnarray*}
\end{proof}

\end{document}